\documentclass{article}
\usepackage{ifthen}
\usepackage{mdwlist}
\usepackage{amsmath,amssymb,amsfonts,amsthm}
\usepackage{bm}
\usepackage{hyperref}
\usepackage{enumitem}
\usepackage{graphicx}
\usepackage{xspace}
\usepackage{verbatim}
\usepackage[margin=1in]{geometry}
\usepackage{color}
\usepackage{thm-restate}
\usepackage{latexsym}
\usepackage{epsfig}
\def\eps{\ve}
\renewcommand{\epsilon}{\ve}
\def\ve{\varepsilon}

\newcommand{\Exp}{\E}
\newcommand{\E}{\mbox{\bf E}}

\newcommand{\pr}[2][]{\mbox{Pr}\ifthenelse{\not\equal{}{#1}}{_{#1}}{}\!\left[#2\right]}

\newtheorem{theorem}{Theorem}

\newtheorem{remark}{Remark}

\newtheorem{lemma}{Lemma}

\newtheorem{definition}{Definition}

\newcommand{\ignore}[1]{}
\providecommand{\poly}{\operatorname*{poly}}

\newcommand{\bg}[1]{\medskip\noindent{\bf #1}}
\definecolor{Red}{rgb}{1,0,0}

\newcommand{\oldbound}[1]{{}}

\newcommand{\Loss}{\text{Loss}}
\newcommand{\Unif}{\text{Uniform}}
\newcommand{\Ind}{\mathbb{I}}
\newcommand{\Supp}{\text{supp}}
\newcommand{\Inv}{\textsc{Inv}}

\usepackage{algorithm2e}
\usepackage{algorithmic}
\title{Actively Avoiding Nonsense in Generative Models}
\author {
Steve Hanneke \\
Princeton, NJ\\
\tt{steve.hanneke@gmail.com}
\and 
Adam Kalai \\
Microsoft Research, New England \\
\tt{adum@microsoft.com}
\and
Gautam Kamath \thanks{Supported by ONR N00014-12-1-0999, NSF CCF-1617730, CCF-1650733, and CCF-1741137. Work partially done while author was an intern at Microsoft Research, New England.}\\
EECS \& CSAIL, MIT \\
\tt{g@csail.mit.edu}
\and
Christos Tzamos \\
Microsoft Research, New England \\
\tt{chtzamos@microsoft.com}
}

\begin{document}

\maketitle

\begin{abstract}
A generative model may generate utter nonsense when it is fit to maximize the likelihood of observed data. This happens due to ``model error,'' i.e., when the true data generating distribution does not fit within the class of generative models being learned. To address this, we propose a model of active distribution learning using a binary invalidity oracle that identifies some examples as clearly invalid, together with random positive examples sampled from the true distribution. The goal is to maximize the likelihood of the positive examples subject to the constraint of (almost) never generating examples labeled invalid by the oracle. Guarantees are agnostic compared to a class of probability distributions. We show that, while proper learning often requires exponentially many queries to the invalidity oracle, improper distribution learning can be done using polynomially many queries. 
\end{abstract}

% !TeX root = main.tex
\section{Introduction}
\label{sec:intro}
Generative models are often trained in an unsupervised fashion, fitting a model $q$ to a set of observed data $x_P \subseteq X$ drawn iid from some true distribution $p$ on $x\in X$. Now, of course $p$ may not exactly belong to family $Q$ of probability distributions being fit, whether $Q$ consists of Gaussians mixture models, Markov models, or even neural networks of bounded size. We first discuss the limitations of generative modeling without feedback, and then discuss our model and results.

%\subsection{Limitations of Generative Modeling from Positive Examples Alone}
Consider fitting a generative model on a text corpus consisting partly of poetry written by four-year-olds and partly of mathematical publications from the {\em Annals of Mathematics}. Suppose that learning to generate a poem that looks like it was written by a child was easier than learning to generate a novel mathematical article with a correct, nontrivial statement. If the generative model pays a high price for generating unrealistic examples, then it may be better off learning to generate children's poetry than mathematical publications. However, without negative feedback, it may be difficult for a neural network or any other model to know that the mathematical articles it is generating are stylistically similar to the mathematical publications but do not contain valid proofs.\footnote{This is excluding clearly fake articles published without proper review in lower-tier venues~\cite{LabbeL13}.} 

As a simpler example, the classic Markovian ``trigram model'' of natural language assigns each word a fixed probability conditioned only on the previous two words. Prior to recent advances in deep learning, for decades the trigram model and its variant were the workhorses of language modeling, assigning much greater likelihood to natural language corpora than numerous linguistically motivated grammars and other attempts~\cite{Rosenfeld00}. However, text sampled from a trigram is typically nonsensical, e.g., the following text was randomly generated from a trigram model fit on a corpus of text from the Wall Street Journal~\cite{JurafskyM09}:
\begin{quote}
They also point to ninety nine point six billion dollars from two hundred
four oh six three percent of the rates of interest stores as Mexico and
gram Brazil on market conditions. 
\end{quote}

In some applications, like text compression using a language model~\cite{WittenNC87}, maximizing likelihood is equivalent to optimizing compression. However, in many  applications involving generation, such nonsense is costly and unacceptable. Now, of course it is possible to always generate valid data by returning random training examples, but this is simply overfitting and not learning. Alternatively, one could incorporate human-in-the-loop feedback such as through crowdsourcing, into the generative model to determine what is a valid, plausible sentence.

In some domains, validity could be determined automatically. Consider a Markovian model of a well-defined concept such as mathematical formulas that compile in \LaTeX{}. Now, consider a $n$-gram Markovian character model which the probability of each subsequent character is determined by the previous $n$ characters. For instance, the expression \$\{2+\{x-y\}\$ is invalid in \LaTeX{} due to mismatched braces. For this problem, a \LaTeX{} compiler may serve as a validity oracle. Various $n$-gram models can be fit which only generate valid formulas. To address mismatched braces, for example, one such model would ensure that it always closed braces within $n$ characters of opening, and had no nested braces. While an $n$-gram model will not perfectly model the true distribution over valid \LaTeX{} formulas, for certain generative purposes one may prefer an $n$-gram model that generates valid formulas over one that assigns greater likelihood to the training data but generates invalid formulas. 

Figure \ref{fig:rectangle} illustrates a simple case of learning a rectangle model for data which is not uniform over a rectangle. A maximum likelihood model would necessarily be the smallest rectangle containing all the data, but most examples generated from this distribution may be invalid. Instead a smaller rectangle, as illustrated in the figure, may be desired.

\begin{figure}[h]\label{fig:rectangle}
  \centering
\includegraphics[width=7cm]{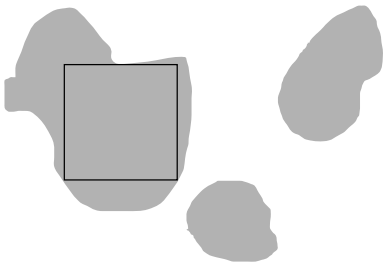}
\caption{Example where the underlying distribution $p$ is uniform over the valid region, shaded in gray. The best valid rectangle corresponding to $q^*$ is outlined on top. }
\end{figure}

Motivated by these observations, we evaluate a generative model $q$ on two axes. First is {\em coverage}, which is related to the probability assigned to future examples drawn from the true distribution $p$. Second is {\em validity}, defined as the probability that random examples generated from $q$ meet some validity requirement. Formally, we measure coverage in terms of a bounded {\em loss}:
$$\Loss(p,q)=\E_{x \sim p}[L(q_x)],$$
where $L:[0,1]\rightarrow [0,M]$ is a bounded decreasing function such as the capped log-loss $L(q_x)=\min(M, \log 1/q_x)$. % or $L(q_x)=\log 1/(q_x+\exp(-M))$. 
A bounded loss has the advantages of being efficiently estimable, and also it enables a model to assign 0 probability to one example (e.g., an outlier or error) if it greatly increases the likelihood of all other data. Validity is defined with respect to a set $V \subseteq X$, and $q(V)$ is the probability that a random example generated from $q$ lies within $V$. 

Clearly, there is a tradeoff between coverage and validity. We first focus on the case of (near) perfect validity. A Valid Generative Modeling (VGM) algorithm if it outputs, for a family of distributions $Q$ over $X$, if it outputs $\hat{q}$ with (nearly) perfect validity and whose loss is nearly as good as the loss of the best valid $q\in Q$. More precisely, $A$ is a VGM learner of $Q$ if for any nonempty valid subset $V \subseteq X$, any probability distribution $p$ over $V$, and any $\eps>0$, $A$ uses $n$ random samples from $p$ and makes $m$ membership oracle calls to $V$ and outputs a distribution $\hat{q}$ such that, $$\Loss(p, \hat{q}) \leq \min_{q \in Q: q(V)=1}\Loss(p,q) + \eps ~\text{ and }~q(V)\geq 1-\eps.$$ 
We aim for our learner to be sample and query efficient, requiring that $n$ and $m$ are polynomial in $M, 1/\eps$ and a measure of complexity of our distribution class $Q$.
Furthermore, we would like our algorithms to be computationally efficient, with a runtime polynomial in the size of the data, namely the $n + m$ training examples. 
A more formal description of the problem is available in Section~\ref{sec:problem}.

$A$ is said to be {\em proper} if it always outputs $\hat{q}\in Q$ and {\em improper} otherwise.
In Section~\ref{sec:impossibility}, we first show that efficient proper learning for VGM is impossible. This is an information-theoretic result, meaning that even given infinite runtime and positive samples, one still cannot solve the VGM problem. Interestingly, this is different from binary classification, where it is possible to statistically learn from iid examples without a membership oracle.

Our first main positive result is an efficient (improper) learner for VGM. The algorithm relies on a subroutine that solves the following {\em Generative Modeling with Negatives} (GMN) problem: given sets $X_P, X_N \subset X$ of positive and negative examples, find the probability distribution $q \in Q$ which minimizes $\sum_{x \in X_P} L(q(x))$ subject to the constraint that $q(X_N)=0$. For simplicity, we present our algorithm for the case that the distribution family $Q$ is finite, giving sample and query complexity bounds that are logarithmic in terms of $|Q|$. However, as we show in Section~\ref{sec:infinite-families}, all of our results extend to infinite families $Q$. It follows that if one has a computationally efficient algorithm for the GMN problem for a distribution family $Q$, then our reduction gives a computationally efficient VGM learning algorithm for $Q$.

Our second positive result is an algorithm that minimizes $\Loss(p,q)$ subject to a relaxed validity constraint comparing against the optimal distribution that has validity $q(V)$ at least $1-\alpha$ for some $\alpha>0$. We show in Section~\ref{sec:partial-validity} that even in this more general setting, it is possible to obtain an algorithm that is statistically efficient but may not be computationally efficient. An important open question is whether there exists a computationally efficient algorithm for this problem when given access to an optimization oracle, as was the case for our algorithm for VGM.

\subsection{Related Work}
\cite{KearnsMRRSS94} showed how to learn distributions from positive examples in the realizable setting, i.e., where the true distribution is assumed to belong to the class being learned. In the same sense as their work is similar to PAC learning~\cite{Valiant84} of distributions, our work is like agnostic learning~\cite{KearnsSS94} in which no assumption on the true distribution is made. 

Generative Adversarial Networks (GANs)~\cite{GoodfellowPMXWOCB14} are an approach for generative modeling from positive examples alone, in which a generative model is trained against a discriminator that aims to distinguish real data from generated data. In some domains, GANs have been shown to outperform other methods at generating realistic-looking examples. Several shortcomings of GANs have been observed~\cite{AroraRZ18}, and GANs are still subject to the theoretical limitations we argue are inherent to any model trained without a validity oracle. 

In supervised learning, there is a rich history of learning theory with various types of queries, including membership which are not unlike our (in)validity oracle. Under various assumptions, queries have been shown to facilitate the learning of complex classes such as finite automata~\cite{Angluin88} and DNFs~\cite{Jackson97}. See the survey of \cite{Angluin92} for further details.  Interestingly, \cite{Feldman09} has shown that for agnostic learning, i.e., without making assumptions on the generating distribution, the addition of membership queries does not enhance what is learnable beyond random examples alone. 
Supervised learning also has a large literature around active learning, showing how the ability to query examples reduces the sample complexity of many algorithms. See the survey of \cite{Hanneke14}. Note that the aim here is typically to save examples and not to expand what is learnable.
 
More sophisticated models, e.g., involving neural networks, can mitigate the invalidity problem as they often generate more realistic natural language and have even been demonstrated to generate \LaTeX{} that nearly compiles~\cite{Karpathy15} or nearly valid Wikipedia markdown. However, longer strings generated are unlikely to be valid. For example, \cite{Karpathy15} shows generated markdown which includes:
\begin{quote}
==Access to ''rap===
The current history of the BGA has been [[Vatican Oriolean Diet]], British Armenian, published in 1893.  While actualistic such conditions such as the [[Style Mark Romanians]] are still nearly not the loss.
\end{quote}

Even ignoring the mismatched quotes and equal signs, note that this example has two so-called ``red links'' to two pages that do not exist. Without checking, it was not obvious to us whether or not Wikipedia had pages titled {\em Vatican Oriolean Diet} or {\em Style Mark Romanians}. In some applications, one may or may not want to disallow red links. In the case that they are considered valid, one may seek a full generative model of what might plausibly occur inside of brackets, as the neural network has learned in this case. If they are disallowed, a model might memorize links it has seen but not generate new ones. A validity oracle can help the learner identify what it should avoid generating.

 In practice, \cite{KusnerPH17} discuss how generative models from neural networks (in particular autoencoders) often generate invalid sequences. 
\cite{JanzWPKH18} learn the validity of examples output by a generative model using oracle feedback. 

% !TeX root = main.tex
%\section{Preliminaries}
%\label{sec:prelim}
\section{Problem Formulation}
\label{sec:problem}
We will consider a setting where we have access to a distribution $p$ over a (possibly infinite) set $X$, and let $p_x$ be the probability mass assigned by $p$ to each $x \in X$.
For simplicity, we assume that all distributions are discrete, but our results extend naturally to continuous settings as well.
Let $\text{supp}(p) \subseteq X$ denote the support of distribution $p$.
We assume we have two types of access to $p$:
\begin{enumerate}
\item Sample access: We may draw samples $x_i \sim p$;
\item Invalidity access: We may query whether a point $x_i$ is ``invalid''.
\end{enumerate}
To be more precise on the second point, we assume we have access to an oracle which can answer queries to the function $\Inv: X \rightarrow \{0, 1\}$, where $\Inv(x) = 1$ indicates that a point is ``invalid.'' As shorthand, we will use $\Inv(q) = \Exp_{x \sim q}[\Inv(x)]$. Put another way, if $V$ is the set of valid points, then $\Inv(q)=1-q(V)$. Henceforth, we find it more convenient to upper-bound invalidity rather than lower-bound validity. % reconciling with the intro 

For this work, we will assume that $\Inv(x)=0$ for all $x \in \text{supp}(p)$, i.e., $\Inv(p)=0$, though examples may also have $\Inv(x)=0$ even if $p(x)=0$. However, we note that it is relatively straightforward to extend our results to a more general case by simply removing from the random positive examples from those that have $\Inv(x_i)=1$.

Our goal is to output a distribution $\hat q$ with low invalidity and expected loss, for some monotone decreasing loss function $L : [0,1] \rightarrow [0,M]$.  In addition to the natural loss function $L(q_x)=\min(M, \log 1/q_x)$ mentioned earlier, a convex bounded loss is $L(q_x)=\log 1/(q_x+\exp(-M))$. 
For a class $Q$ of candidate distributions $q$ over $X$, we aim to solve the following problem:
$$\min_{ q \in Q \atop \Inv(q) = 0 } \Loss(q)=\min_{ q \in Q \atop \Inv(q) = 0 } \Exp_{x \sim p} \left[ L(q_x) \right].$$
Let $OPT$ be the minimum value of this objective function, and $q^*$ be a distribution which achieves this value. In practice we can never determine with certainty whether any $\hat q$ has 0 invalidity. Instead, given $\eps_1, \eps_2 > 0$, we want that $\Loss(\hat q) \le OPT+\eps_1$ and $\Inv(\hat q) \le \eps_2$. 

%We consider a generalization of this problem which allows partial validity in Section~\ref{sec:partial-validity}.

\begin{remark}
  Note that given a candidate distribution $\hat q$ it is straightforward to check whether it satifies the loss and validity requirements, with probability $1-\delta$, by computing the empirical loss using $O\left(\frac 1 {\eps_1^2} \log(1/\delta)\right)$ samples from $p$ and by querying the invalidity oracle $O\left(\frac 1 {\eps_2} \log(1/\delta)\right)$ times using samples generated from $\hat q$. This observation allows us to focus on distribution learning algorithms that succeed with a constant probability as we can amplify the success probability to $1-\delta$ by repeating the learning process $O(\log(1/\delta))$ times and checking whether the ouput is correct. 
\end{remark}

\section{Proper Learning}
\label{sec:proper}
For ease of exposition, we begin with a canonical and simple example, where our goal is to approximate the distribution $p$ using a 
uniform distribution over a two-dimensional rectangle (or, in higher dimensions, a multi-dimensional box).

Here, the goal is to find a uniform distribution $q^*$ over a rectangle that best approximates
$p$ (i.e., minimizes some loss) while lying entirely in its valid region. 
We are allowed to output a uniform distribution $\hat q$ over a rectangle that has at least $1-\eps_2$ of its mass within the valid region.
Figure~\ref{fig:rectangle} illustrates the target distribution $q^*$ graphically.

\subsection{Example: Uniform distributions over a Box}

Let $X = \{0,1,...,\Delta-1\}^d$ and assume that $Q$ is the family of distributions that are uniform over a box, i.e.
for every $q \in Q$, there exists $\vec a, \vec b \in \{0,1,...,\Delta-1\}^d$ such that:
$$q_x = \frac { \Ind[ \forall i \in \{1,...,d\}: x_i \in [a_i, b_i] ] } { \prod_{i=1}^d (b_i-a_i+1) }$$

\begin{theorem}
  Using $O\left(\frac {d M^2} {\eps_1^2} \right)$ samples and $\frac {1} {\eps_2} \left(\frac {d M} {\eps_1} \right)^{O( d )}$ invalidity queries on $p$, there exists an algorithm which identifies a distribution $\hat q \in Q$, such that $\Inv(\hat q) \le \eps_2$ and 
  $\Loss(\hat q) \le \Loss(q^*) + \eps_1$ with probability $3/4$.
\end{theorem}

\begin{proof}
  Since the VC-dimension of $d$-dimensional boxes is $2 d$, with probability $7/8$ after taking a set $X_P$ of $P = O\left(\frac {d M^2} {\eps_1^2} \right)$ samples from $p$, we can estimate 
  $p(\Supp(q))$ for all distributions $q \in Q$ within $\pm \frac {\eps_1} {2 M}$ by forming the empirical distribution. This implies that the empirical loss $\overline \Loss(q) = \frac 1 { |X_P| } \sum_{x \in X_p} L(q_x)$ is an estimate to the loss function, i.e. $\overline \Loss(q) \in \Loss(q)\pm \frac {\eps_1} {2}$.
  
  Now consider the optimal distribution $q^*$. Observe that any distribution $q \in Q$, such that $\Supp(q) \subseteq \Supp(q^*)$ and $\Supp(q)\cap X_P = \Supp(q^*)\cap X_P$, satisfies $\overline \Loss(q) \le \overline \Loss(q^*)$ and $\Inv(q)=0$. Thus, there exists a $q'\in Q$ with this property that has at least one point $x \in X_P$ in each of the $2 d$ sides of its box.
  
  As there are at most $P^{2 d}$ such boxes, we can check identify which of their corresponding distribution $q \in Q$ have $\Inv(q) \le \eps_2$ by quering $\Inv$ at $O\left( \frac 1 {\eps_2} \log \left( P^{2 d} \right) \right)$ random points from each of them. This succeeds with probability $7/8$ and uses in total 
  $\frac {1} {\eps_2} \left(\frac {d M} {\eps_1} \right)^{O( d )}$ 
  invalidity queries. 
  
  We pick $\hat q$ to be the distribution that minimizes the empirical $\overline \Loss(\hat q)$ out of those that have no invalid samples in the support.
  Overall, with probability $3/4$, we have that $\Inv(\hat q) \le \eps_2$ and
  $$
    \Loss(\hat q) \le \overline \Loss(\hat q) + \frac {\eps_1} 2 
    \le \overline \Loss(q') + \frac {\eps_1} 2 
    \le \overline \Loss(q^*) + \frac {\eps_1} 2
    \le \Loss(q^*) + {\eps_1}.
$$
\end{proof}

\subsection{Impossibility of Proper Learning}\label{sec:impossibility}

The example in the previous section required number of queries that is exponential in $d$ in order to output a distribution $\hat q \in Q$ with $\Inv(\hat q) \le \eps_2$ and $\Loss(\hat q) \le \Loss(q^*) + \eps_1$. 
We show that such an exponential dependence in $d$ is required when one aims to learn a distribution $\hat q$ properly even for the class of uniform distributions over axis-parallel boxes.

\begin{theorem}\label{thm:rectangleslb}
  Even for $\Delta = 2$, the number of queries required to find a distribution $\hat q \in Q$ such that $\Inv(\hat q) \le \frac 1 4$ and 
  $\Loss(\hat q) \le \Loss(q^*) + \frac 1 {2 d}$ with probability at least $3/4$ is at least $2^{\Omega(d)}$.
\end{theorem}
\begin{proof}
  We describe the construction of the lower-bound below:
  \begin{itemize}
  \item The distribution $p$ assigns probability $1/d$ to each standard basis vector $\vec {e_i}$, i.e., the vector with $i$-th entry equal to $1$ and all other coordinates equal to $0$.
  \item For some arbitrary vector $y \in \{0,1\}^d$ with $|y| = \sum_{i=1}^d y_i = d/3$, we define $\Inv(x)$ as: $$\Inv(x) = \begin{cases}
       0 &\quad\textbf{if } |x| < d/6 \textbf{ or } \text{for all } i, x_i \le y_i\\
       1 &\quad\text{otherwise.} \\ 
     \end{cases}$$
  \item The loss function is the coverage function, i.e., $L(q_x) = \Ind[ q_x = 0 ]$, where we pay a loss of $1$ for each point $q$ assigns $0$ mass to, and $0$ otherwise.
\end{itemize}

Given this instance, the optimal $q^*$ is uniform over the box $\times_{i=1}^d \{0,y_i\}$ and has loss $\frac 2 3$. In order to achieve loss $\frac 23 + \frac 1 {2d}$, the output distribution $\hat q$ must include at least $d/3$ of the vectors $\vec e_i$ in its support. Thus, $\hat q$ must be a box $\times_{i=1}^d \{0,y'_i\}$ defined by some vector $y' \in \{0,1\}^d$ with 
$|y'| \ge d/3$. Moreover, it must be that $y' = y$. This is because if there exists a coordinate $j$ such that $y'_j=1$ and $y_j=0$, then with probability greater than $1/4$, the distribution $q$ produces a sample $x$ with $x_j=1$ and $|x| \ge d/6$. Since such a sample is invalid, $\Inv(\hat q) > \frac 14$ which would lead to a contradiction.

Therefore the goal is to find the vector $y$. 
Since any samples from $p$ only produce points $e_i$ they provide no information about $y$.
Furthermore, queries to $\Inv$ at points $x$ with $|x| < d/6$ or $|x| > d/3$ also provide no information about $y$, as in the former case $\Inv(x) = 0$ since $|x| < d/6$, and in the latter case $\Inv(x) = 1$ since there will always be an $i$ where $1 = x_i > y_i = 0$. 
Therefore, it only makes sense to query points with $|x| \in [d/6,d/3]$.

We show that the number of queries needed to identify the true $y$ is exponential in $d$. 
We do this with a Gilbert-Varshamov style argument.
To see this, consider a set of vectors $Y \subset \{0,1\}^d$ such that for all $y' \in Y$ we have that $|y'|=d/3$ and any two distinct vectors $y^1, y^2 \in Y$ have fewer than $d/6$ coordinates where they are both 1, i.e. $\sum_i y^1_i \cdot y^2_i < d/6$.

Given this set $Y$, note that any query to $\Inv$ at a point $x$ with $|x| \in [d/6,d/3]$ eliminates at most a single $y' \in Y$.
Thus with fewer than $|Y|/2$ queries, the probability that the true $y$ is identified is less than $1/2$.

To complete the proof, we show that a set $Y$ exists with $|Y| = e^{d/216}$. We will use a randomized construction where we pick $|Y|$ random points $y^1,...,y^{|Y|} \in \{0,1\}^d$ with $|y^a|=d/3$ uniformly at random. Consider two such random points $y^a$ and $y^b$. 

Define the random variable $z_i$ to be 1 if $y^1_i = y^2_i = 1$ and 0 otherwise. We have
$$\Pr [z_i = 1] = \frac{1}{3} \cdot \frac{1}{3} = \frac{1}{9}.$$
Although $z_i$'s are not independent, they are negative correlated. We can apply the multiplicative Chernoff bound:
$$
\Pr\left[\sum_{i=1}^d z_i \ge d/6\right] \leq e^{-d/108}
$$

Then by a union bound over all pairs $a<b$, we have
\[
\Pr[\forall 1 \leq a < b \leq |Y|, \sum_i y^a_i \cdot y^b_i < d/6 ] > 1- \binom{|Y|}{2} \cdot e^{-d/108} > 0. 
\]
This shows that the number of queries an algorithm must make to succeed with probability at least $3/4$ is at least $2^{\Omega(d)}$.
\end{proof}

As Theorem~\ref{thm:rectangleslb} shows, proper learning suffers from a ``needle in a haystack'' phenomenon. To build intuition, we present an alternative simpler setting that illustrates this point more clearly.

Let $Q$ be the set of all distributions $q_i$ that, with probability $\frac12$, output $0$, and otherwise output $i>0$. 
Let $p$ be the distribution that always outputs $0$ and suppose that $\Inv(i)=1$ for all $i \neq \{0,i^*\}$ for some arbitrary $i^*$. 
In order to properly learn the distribution $\hat q$, one needs to locate the hidden $i^*$ by querying the invalidity oracle many times. This requires a number of queries that is proportional to the size of the domain $X$, which is intractable when the domain is large (e.g., in high dimensions) or even infinite.

Note, however, that in this example, even though learning a distribution $q$ within the family $Q$ is hard, we can easily come up with an improper distribution that always outputs point $0$. Such a distribution is always valid and achieves optimal loss. In the next section we show that even though proper learning may be information-theoretically expensive or impossible, it is actually always possible to improperly learn using polynomially many samples and invalidity queries.

% !TeX root = main.tex
\section{Improper Learning}
\label{sec:improper}

In this section, we show that if we are allowed to output a distribution that is not in the original family $Q$, we can efficiently identify a distribution that achieves close to optimal loss and almost-full validity using only polynomially many samples from $p$ and invalidity queries.

\subsection{Algorithm}

We provide an algorithm, Algorithm \ref{alg:full-validity}, that can solve the task computationally efficiently assuming access to an optimization oracle $\text{Oracle}(X_P,X_N)$. $\text{Oracle}(X_P,X_N)$ takes as input sets $X_P$ and $X_N$ of positive and negative (invalid) points and outputs a distribution $q$ from the family of distributions $Q$ that minimizes the empirical loss with respect to $X_P$ such that $\Supp(q) \cap X_N = \emptyset$, i.e. no negative point in $X_N$ is in the support of $q$.

\begin{algorithm}[ht]
   \caption{Improperly learning to generate valid samples}
   \label{alg:full-validity}
\begin{algorithmic}[1]
   \STATE {\bfseries Input:} Distribution family $Q$, sample and invalidity access to $p$, and parameters $\ve_1, \ve_2 > 0$.
   \STATE Draw a set $X_P$ of $P$ samples from $p$.
   \STATE Set $X_N \leftarrow \emptyset$
   \FOR{$i=1,...,R$}
     \STATE Let $q^i \leftarrow \text{Oracle}(X_P,X_N)$.
     \STATE Generate $T$ samples from $q^i$ and query the invalidity of each of them.
     \STATE Let $x^-_1,...,x^-_k$ be the invalid samples.
     \IF{there are no invalid samples, i.e. $k=0$}
        \RETURN $q^i$
     \ELSE
       \STATE Set $X_N \leftarrow X_N \cup \{x^-_1,...,x^-_k\}$
     \ENDIF
   \ENDFOR
   \STATE Sample $i \sim \Unif(\{1,...,R\})$
   \STATE Let $A^i \leftarrow {\{x: \exists j > i \text{ with } x \in \Supp(q^j)\}}$
   \RETURN the distribution that samples $x \sim q^i$ and outputs $x$ if $x \in A^i$ and any valid point $x^*$ o/w
\end{algorithmic}
\end{algorithm}

The algorithm repeatedly finds the distribution with minimum loss that doesn't contain any of the invalid points seen so far and tests whether it achieves almost full-validity. If it does, then it outputs that distribution. Otherwise it tries again using the new set of invalid points. However, this process could repeat for a very long time without finding a distribution. To avoid this, after running for a few rounds, if it has failed to output a distribution, the algorithm is able to generate an improper distribution that provides the required guarantee to solve the task. This meta-distribution is obtained by randomly picking one of the candidate distributions examined so far and filtering out points that no other distributions agree on.

\subsection{Analysis}

We show that this Algorithm~\ref{alg:full-validity} outputs with high probability a distribution $\hat q$ that has $\Loss(\hat q) \le \Loss(q^*) + \eps_1$ and $\Inv(\hat q) \le \eps_2$. 

\begin{theorem}\label{thm:full-validity}
  The choice of parameters
  
  \begin{equation}
    P=\Theta\left(\frac{M^2}{\ve_1^2} \log |Q| \right), \quad R = \Theta \left( \frac { M } {\eps_1} \right), \quad  T = \Theta\left(\frac{R}{\ve_2} \log |Q| \right)
    \end{equation}
  guarantees that Algorithm~\ref{alg:full-validity} outputs w.p. $3/4$ a distribution $\hat q$ with $\Loss(\hat q) \le \Loss(q^*) + \eps_1$ and $\Inv(\hat q) \le \eps_2$ using $\Theta\left(\frac{M^2}{\ve_1^2} \log |Q| \right)$ samples from $p$ and $\Theta\left(\frac{M^2}{ \ve_1^2 \ve_2} \log |Q|\right)$ invalidity queries.
  
  The algorithm runs in time polynomial in $M$, $\ve_1^{-1}$, $\ve_2^{-1}$, and $\log |Q|$ assuming that the following each can be  performed at unit cost: (a) queries to \text{Oracle}, (b) sampling from the distributions output by $\text{Oracle}$, and (c) checking whether a point $x$ is in the support of a distribution output by $\text{Oracle}$.
\end{theorem}

Of course, the success probability can be boosted from 3/4 to arbitrarily close to $1-\delta$ by repeating the algorithm $O(\log 1/\delta)$ times and taking the best output. We prove Theorem~\ref{thm:full-validity} by showing two lemmas, Lemma~\ref{lem:full-validity-inv} and Lemma~\ref{lem:full-validity-loss}, bounding the invalidity and loss of the returned distribution.

\begin{lemma}\label{lem:full-validity-inv}
The returned distribution $\hat q$ by Algorithm~\ref{alg:full-validity} satisfies $\Inv(\hat q) \le \eps_2$ w.p. $7/8$.
\end{lemma}

\begin{proof}
  Let $ \text{Invalid} = \{x : \Inv(x) = 1 \}$ be the set of invalid points.
  Consider $q^i$ for some $i$ and any distribution $q \in Q$. If $q^i( \Supp(q) \cap \text{Invalid} ) \ge \frac { \eps_2 } { R }$, then with probability at least $\frac { \eps_2 } { R }$ a sample generated from $q^i$ lies in $\Supp(q) \cap \text{Invalid}$. Thus, with $T = \Theta( \frac  { R } { \eps_2 } \log |Q| )$ samples at least one lies in $\Supp(q) \cap \text{Invalid}$ w.p. $1 - \frac 1 {8 |Q| R}$. By a union bound for all $i$ and $q \in Q$, we get that with probability $7/8$ for all $q_i$ and all distributions $q \in Q$, if $q^i( \Supp(q) \cap \text{Invalid} ) \ge \frac { \eps_2 } { R }$ then at least one of the $T$ samples drawn from $q^i$ lies in $\Supp(q) \cap \text{Invalid}$.
  We therefore assume that this holds.
  
  Then, if the returned distribution $\hat q = q^i$ for some $i$, we get $$\Inv(q^i) = q^i( \Supp(q^i) \cap \text{Invalid} ) < \frac { \eps_2 } { R } \le \eps_2 $$ as required.
  To complete the proof we show the required property when returned distribution $\hat q$ is the improper meta-distribution.
  
  We have that for all $j > i$, $q^i( \Supp(q^j) \cap \text{Invalid} ) < \frac { \eps_2 } { R }$ since after round $i$ for any $q \in Q$ with $q^i( \Supp(q) \cap \text{Invalid} ) \ge \frac { \eps_2 } { R }$ the set $X_N$ will contain at least one point in $\Supp(q) \cap \text{Invalid}$ and thus any such $q$ will not be considered.
  
  Therefore, we have that 
  \begin{align*}
    \Inv(\hat q) 
&= \Exp_{x\sim\hat q}\left[ \Inv(x) \right] \\
&= \Exp_{x\sim q^i}\left[ \Inv(x) \cdot \Ind\left[\exists j > i: x \in \Supp(q^j) \right]\right] \\
&\le  \sum_{j=i+1}^R \Exp_{x\sim q^i}\left[ \Inv(x) \cdot \Ind\left[ x \in \Supp(q^j) \right] \right] \\  
&= \sum_{j=i+1}^R q^i( \Supp(q^j) \cap \text{Invalid} ) \le  \sum_{j=i+1}^R \frac { \eps_2 } { R } <  \eps_2.
  \end{align*}
\end{proof}

\begin{lemma}\label{lem:full-validity-loss}
The returned distribution $\hat q$ by Algorithm~\ref{alg:full-validity} satisfies $\Loss(\hat q) \le \Loss(q^*) + \eps_1$ w.p. $7/8$.
\end{lemma}
\begin{proof}
  Since we draw $P=\Theta\left(\frac{M^2}{\ve_1^2} \log |Q| \right)$ samples from $p$, we have that the empirical loss $\overline \Loss(q) \in \Loss(q) \pm \frac {\eps_1} 4$ for all $q \in Q$ with probability $1-1/16$. We thus assume from here on that this is true.
  
  In that case, must be that $\overline \Loss(q^i) \le \overline \Loss(q^*)$. This is because the algorithm terminates if $q^i = q^*$ since $q^*$ generates no invalid samples and no $q^i$ with $\overline \Loss(q^i) > \overline \Loss(q^*)$ will be considered before examining $q^*$.
  
  This implies that at any point, we have that $\Loss(q^i) \le \overline \Loss(q^i) + \frac {\eps_1} 4 \le \overline \Loss(q^*) + \frac {\eps_1} 4 \le \Loss(q^*) + \frac {\eps_1} 2$. 
  
  Therefore, in the case that the distribution that is output is $\hat q = q^i$ it will satisfy the given condition.
  To complete the proof we show the required property when returned distribution $\hat q$ is the improper meta-distribution.
  
  In that case, we have that for any $i \in [R]$:
    \begin{align*}
      \Loss(\hat q) 
  &\le \Exp_{x\sim p}\left[ L\left(q_x^i \cdot \Ind\left[\exists j > i: x \in \Supp(q^j) \right] \right) \right] \\
  &\le \Loss(q^i) + M \cdot \Pr_{x\sim p}\left[ x \in \Supp(q^i) \wedge \forall j > i: x \notin \Supp(q^j) \right] \\
  &\le \Loss(q^*) + \frac {\eps_1} 2 + M \cdot \Pr_{x\sim p}\left[ x \in \Supp(q^i) \wedge \forall j > i: x \notin \Supp(q^j) \right]
    \end{align*}
However, since a random index $i \sim \Unif(\{1,...,R\})$ is chosen, we have that in expectation over this random choice
    \begin{align*}
\Exp_i&\left[ \Pr_{x\sim p}\left[ x \in \Supp(q^i) \wedge \forall j > i: x \notin \Supp(q^j) \right] \right] \\
&\le \frac 1 R \sum_{i=1}^R \Pr_{x\sim p}\left[ x \in \Supp(q^i) \wedge \forall j > i: x \notin \Supp(q^j) \right] \\
&\le \frac 1 R \Exp_{x\sim p} \left[ \sum_{i=1}^R \Ind\left[ x \in \Supp(q^i) \wedge \forall j > i: x \notin \Supp(q^j) \right] \right] 
\le \frac 1 R
   \end{align*}
  where the last inequality follows since $\sum_{i=1}^R \Ind\left[ x \in \Supp(q^i) \wedge \forall j > i: x \notin \Supp(q^j) \right] \le 1 $ as only the largest $i$ with $x \in \Supp(q^i)$ has that for all $j > i$, $x \notin \Supp(q^j)$.

By Markov's inequality, we have that with probability $1-1/16$, a random $i$ will have $$\Pr_{x\sim p}\left[ x \in \Supp(q^i) \wedge \forall j > i: x \notin \Supp(q^j) \right] \le \frac {16} R.$$

Therefore, the choice of $R = 32 \frac M {\ve_1} = \Theta \left( \frac { M } {\eps_1} \right)$ guarantees that $\Loss(\hat q) \le \Loss(q^*) + {\eps_1}$.
The overall failure probability is at most $1/16+ 1/16 = 1/8$.

\end{proof}

% !TeX root = main.tex
\section{Extensions}
\label{sec:extensions}

% !TeX root = main.tex
\subsection{Partial validity}
\label{sec:partial-validity}
In this section, we consider a generalization of our main setting, where we allow some slack in the validity constraint.
More precisely, given some parameter $\alpha > 0$, we now have the requirement that $\Loss(\hat q) \leq \Loss(q^*) + \eps_1$ and $\Inv(\hat q) \leq \alpha + \eps_2$, where $q^*$ is the optimal distribution which minimizes $\Loss(q^*)$ such that $\Inv(q^*) \leq \alpha$.

%In this result, we allow points to be ``partially valid'' -- specifically, we let $\Inv: X \rightarrow [0,1]$ take fractional values.

\subsubsection{Algorithm}
We provide an algorithm for solving the partial validity problem in Algorithm~\ref{alg:partial-validity}.
This method is sample-efficient, requiring a number of samples which is $\poly\left(M, \eps_1^{-1}, \eps_2^{-1}, \log |Q|\right)$.

\begin{algorithm}[ht]
   \caption{Learning a distribution with partial validity}
   \label{alg:partial-validity}
\begin{algorithmic}[1]
   \STATE {\bfseries Input:} Sample and invalidity access to a distribution $p$, parameters $\ve_1, \ve_2, \alpha > 0$, a family of distributions $Q$.
   \STATE Using $n_1$  samples from $p$, empirically estimate $\overline \Loss(q) \in \Loss(q) \pm \frac {\eps_1} 3$ for all $q \in Q.$
   \FOR{$\ell \in \left\{0, \frac {\eps_1} 3,..., M\right\}$} \label{ln:partial-validity-outer-loop}
   \STATE Let $D = \{q \in Q\ |\ \overline \Loss(q) \le \ell \}$.
   \STATE Let $x^*$ be any point with $\Inv(x^*) = 0$.
   \STATE Let $\mu_D$ be the distribution which samples a distribution $q$ uniformly from $D$, and then draws a sample from $q$.
   \WHILE{$D \neq \emptyset$} \label{ln:partial-validity-inner-loop}
   \STATE Draw $n_2$ samples $x_1, ..., x_{n_2}$ from $\mu_D$.
   \IF {$\frac{1}{n_2} \sum_{i=1}^{n_2} \Inv(x_i) \Pr_{q \sim \Unif(D)}[q(x_i) {\eps_1} < 3 \mu_D(x_i)  {M} ] \le \alpha + \frac {4 \eps_2}{5}$}
   \RETURN $\mu'_D$, which samples $x$ from $\mu_D$ with probability $$\Pr_{q \sim \Unif(D)}[q(x) {\eps_1} < 3 \mu_D(x) {M} ],$$ and samples $x^*$ otherwise.
   \ELSE \STATE Remove all distributions $q$ from $D$ for which $$\frac{1}{n_2} \sum_{i=1}^{n_2} \Inv(x_i) \frac {q(x_i)} {\mu_D(x_i)} \Ind[q(x_i) {\eps_1} < 3 \mu_D(x_i) {M} ] > \alpha + \frac {\eps_2} {5}.$$  
   \ENDIF
   \ENDWHILE
   \ENDFOR
\end{algorithmic}
\end{algorithm}

\subsubsection{Analysis}
We will show that, with high probability, Algorithm~\ref{alg:partial-validity} outputs a distribution $\hat q$ that has $\Loss(\hat q) \leq \Loss(q^*) + \eps_1$ and $\Inv(\hat q) \leq \alpha + \eps_2$.

\begin{theorem}\label{thm:partial-validity}
  Suppose that the loss function $L$ is convex.
  The choice of parameters
  \begin{equation}
  n_1 = \Theta\left(\frac{M^2}{\ve_1^2} \log |Q| \right), n_2 = \Theta\left(\frac {M^2} {\eps_1^2 \eps_2^2} \log |Q| \log \left(\frac{M \log |Q|}{\eps_1 \eps_2}\right)\right)
  \end{equation}
  guarantees that Algorithm~\ref{alg:partial-validity} outputs w.p. $3/4$ a distribution with $\Loss(\hat q) \le \Loss(q^*) + \eps_1$ and $\Inv(\hat q) \le \alpha + \eps_2$ using $\Theta\left(\frac{M^2}{\ve_1^2} \log |Q| \right)$ samples from $p$ and $\Theta\left(\frac{M^3}{ \ve_1^3 \ve_2^3} \log^2 |Q| \log \left(\frac{M \log |Q|}{\eps_1\eps_2}\right)\right)$ invalidity queries.
\end{theorem}

\begin{remark}
We note that this algorithm still works in the case where points may be ``partially valid'' -- specifically, we let $\Inv: X \rightarrow [0,1]$ take fractional values.
This requires that we have access to some point $x^*$ where $\Inv(x^*) = 0$, which we assume is given to us by some oracle.
For instance, the distribution may choose to output a dummy symbol $\bot$, rather than output something which may not be valid. 
\end{remark}

We prove Theorem~\ref{thm:partial-validity} through three lemmas.
The sample complexity bound follows from the values of $n_1$, $n_2$, the fact that we have at most $O\left(\frac{M}{\eps_1}\right)$ iterations of the loop at Line~\ref{ln:partial-validity-outer-loop}, and Lemma~\ref{lem:partial-validity-inner-loop} which bounds the number of iterations of the loop at Line~\ref{ln:partial-validity-inner-loop} as $O\left(\frac{\log |Q|}{\eps_2}\right)$ for any $\ell$.
To argue correctness, Lemmas~\ref{lem:partial-validity-invalid} and~\ref{lem:partial-validity-loss} bound the invalidity and loss of any output distribution, respectively.

\begin{lemma}
\label{lem:partial-validity-inner-loop}
With probability at least $14/15$, the loop at Line~\ref{ln:partial-validity-inner-loop} requires at most $O\left(\frac{\log |Q|}{\eps_2}\right)$ iterations for each $\ell$.
\end{lemma}
\begin{proof}
To bound the number of iterations, we will show that if no distribution is output, $|D|$ shrinks by a factor $1-\frac  {\eps_2}  {5}$. 
As we start with at most $|Q|$ candidate distributions, this  implies the required bound.

We note that we have a multiplicative term $\log \left(\frac{M \log |Q|}{\eps_1 \eps_2}\right)$ in the expression for $n_2$.
This corresponds to certain estimates being accurate for the first $\poly(M, \log |Q|, \eps_1^{-1}, \eps_2^{-1})$ times they are required by a union bound argument.
As this proof will justify, each line in the algorithm is run at most $\frac{M \log |Q|}{\eps_1\eps_2}$ times.
Thus, for ease of exposition, we simply will state that estimates are accurate for every time the line is run.

We thus need to count how many candidate distributions in $D$ are eliminated in every round given that the empirical invalidity of $\mu'_D$ is at least $\alpha + \frac {4 \eps_2}{5}$, i.e.
$$\frac 1 N \sum_{i=1}^N \Inv(x_i) \Pr_{q \sim \Unif(D)}[q(x_i) {\eps_1} < 3 \mu_D(x_i)  {M} ] > \alpha + \frac {4 \eps_2}{5}.$$
This implies that the true invalidity of $\mu'_D$ is at least $\alpha + \frac {3 \eps_2}{5}$:
since $n_2 = \Omega\left(\frac{1}{\eps_2^2} \cdot \log\left(\frac{M \log |Q|}{\eps_1\eps_2}\right)\right)$, we have that $\overline \Inv(\mu'_D) = \Inv(\mu'_D) \pm \frac {\eps_2}{5}$ each time this line is run, with probability $29/30$.

Similarly, for every $q$ we have that the estimator $\frac{1}{n_2}  \sum_{i=1}^{n_2} \Inv(x_i) \frac {q(x_i)} {\mu_D(x_i)} \Ind[q(x_i) {\eps_1} < 3 \mu_D(x_i) {M} ]$ is an accurate estimator for the validity of $q'$ which is the distribution that generates a sample $x$ from $q$ and returns $x$ if $q(x)  {\eps_1} \le 3 \mu_D(x)  {M}$ and $x^*$ otherwise. This is because, since $\Inv(x^*)=0$, we have
$$\begin{aligned}
\Exp_{x \sim \mu_D} \left[ \Inv(x) \frac {q(x)} {\mu_D(x)} \Ind[q(x) {\eps_1} < 3 \mu_D(x) {M} ] \right] &= \Exp_{x \sim q} \left[ \Inv(x) \Ind[q(x) {\eps_1} < 3 \mu_D(x) {M} ] \right]\\
&= \Exp_{x \sim q'} \left[ \Inv(x) \right] = \Inv(q').
\end{aligned}$$

Note that our estimate $\overline \Inv(q')$ is the empirical value
$$\frac{1}{n_2} \sum_{i=1}^{n_2} \Inv(x_i) \frac {q(x_i)} {\mu_D(x_i)} \Ind[q(x_i) {\eps_1} < 3 \mu_D(x_i) {M} ],$$
where
$$\frac{q(x_i)}{\mu_D(x_i)} \Ind[q(x_i) {\eps_1} < 3 \mu_D(x_i) {M} ] \leq \frac{3M}{\eps_1}.$$
Since we are estimating the expectation of a function upper bounded by $O(M/\eps_1)$ and there are at most $|Q|$ distributions $q'$ at each iterations, $n_2 = \Omega\left(\frac {M^2} {\eps_1^2 \eps_2^2} \log |Q| \log \left(\frac{M \log |Q|}{\eps_1 \eps_2}\right)\right)$ samples are sufficient to have that the empirical estimator $\overline \Inv(q') = \Inv(q') \pm \frac {\eps_2}{5}$ for all distributions $q'$ considered and all times this line is run, with probability $29/30$. 
Thus, it is sufficient to count how many $q \in D$ exist with $\Inv(q') > \alpha + \frac {3 \eps_2}{5}$.

To do this, we notice that $\Exp_{q \in \Unif(D)} [\Inv(q')] = \Inv(\mu'_D) > \alpha + \frac {3 \eps_2}{5}$. Then, as $\Inv(q') \le 1$, we have that $\Pr_{q\sim \Unif(D)}[\Inv(q') > \alpha + \frac {2 \eps_2}{5}] \ge \frac {\eps_2}{5}$. This yields the required shrinkage of the set $D$.
\end{proof}

\begin{lemma}
\label{lem:partial-validity-invalid}
With probability at least $14/15$, if at any step a distribution $\mu'_D$ is output, $\Inv(\mu'_D) \le \alpha + \eps_2$.
\end{lemma}
\begin{proof}
The estimator $\frac{1}{n_2} \sum_{i=1}^{n_2} \Inv(x_i) \Pr_{q \sim \Unif(D)}[q(x_i) {\eps_1} < 2 \mu_D(x_i) {M} ]$ estimates the empirical fraction of samples that are invalid for distribution $\mu'_D$. 
Since $n_2 = \Omega\left( \frac {1} {\eps_2^2} \log \left(\frac{M \log |Q|}{\eps_1 \eps_2}\right) \right)$, and by Lemma~\ref{lem:partial-validity-inner-loop} each line is run at most $O\left(\frac{M \log |Q|}{\eps_1\eps_2}\right)$ times, the empirical estimate of $\overline \Inv(\mu'_D) = \Inv(\mu'_D) \pm \frac {\eps_2} 5$ for all iterations, with probability at least $14/15$.
The statement holds as $\mu'_D$ is only returned if the estimate for the invalidity of $\mu'_D$ is at most $\alpha + \frac {4 \eps_2}{5}$.
\end{proof}

\begin{lemma}
\label{lem:partial-validity-loss}
With probability at least $14/15$, if at any step a distribution $\mu'_D$ is output, $\Loss(\mu'_D) \le \ell + 2\eps_1/3$, where $\ell$ is the step at which the distribution was output.
\end{lemma}
\begin{proof}
For any $q \in D$ denote by $q'$ the distribution that generates a sample $x$ from $q$ and returns $x$ if $q(x)  {\eps_1} \le 3 \mu_D(x)  {M}$ and $x^*$ otherwise.
Notice that $\mu'_D(x) = \Exp_{q \sim \Unif(D)} [ q'(x) ]$.
We have that
$$\begin{aligned}
\Loss(\mu'_D) &= \Exp_{x \sim p} [ L(\mu'_D(x)) ] \le \Exp_{x \sim p} [ \Exp_{q \sim \Unif(D)} [ L(q'(x)) ] ] \\
&\le \Exp_{x \sim p \atop q \sim \Unif(D)} \left[ L(q(x)) + M \cdot \Ind[q(x)  {\eps_1} > 3 \mu_D(x)  {M}]  \right]\\
&\le \sup_{q \in D} \Loss(q) + M \cdot \Pr_{x \sim p \atop q \sim \Unif(D)}[q(x)  {\eps_1} > 3 \mu_D(x)  {M}]\\
\end{aligned}$$
The equality is the definition of $\Loss$, the first inequality uses convexity of $L$ and Jensen's inequality, and the second inequality uses the fact that $L(\cdot) \leq M$.

However, for any given $x$, we have that $\Exp_{q \sim \Unif(D)}[ q(x) ] =  \mu_D(x)$ and thus by Markov's inequality we obtain that for all $x$ $$\Pr_{q \sim \Unif(D)}[q(x)  {\eps_1} > 3 \mu_D(x)  {M}] \le \frac {\eps_1} {3 M}.$$
This implies that $M \cdot \Pr_{x \sim p \atop q \sim \Unif(D)}[q(x)  {\eps_1} > 3 \mu_D(x)  {M}]$ is at most $\frac {\eps_1} {3}$.
To complete the proof we note that $\sup_{q \in D} \Loss(q)$ is at most $\ell + \frac {\eps_1} {3}$: 
since we are estimating the mean of $L(\cdot)$ which is bounded by $M$, there are $|Q|$ distributions $q$ which are considered, and $n_1 = \Omega\left(\frac{M^2}{\eps_1^2}\log |Q|\right)$, the statement holds for all $q$ simultaneously with probability at least $14/15$.
\end{proof}

The proof of Theorem~\ref{thm:partial-validity} concludes by observing that the optimal distribution $q^*$ is never eliminated (assuming all estimates involving its loss and validity are accurate, which happens with probability at least $19/20$), and that the loop in line~\ref{ln:partial-validity-outer-loop} steps by increments of $\eps_1/3$. 
Combining this with Lemma~\ref{lem:partial-validity-loss}, if we output $\hat q$, then $\Loss(\hat q) \leq \Loss(q^*) + \eps_1$.

\newcommand{\fat}{s}
\newcommand{\vc}{d}

\subsection{General Densities}
\label{sec:densities}

For simplicity of presentation, we have formulated the above results in terms of probability mass functions $q$ on a discrete domain $X$.
However, we note that all of the above results easily extend to general density functions on an abstract measurable space $X$, which 
may be either discrete or uncountable.  Specifically, if we let $\mu_{0}$ denote an arbitrary reference measure on $X$, 
then we may consider the family $Q$ to be a set of \emph{probability density functions} $q$ with respect to $\mu_{0}$: that is, 
non-negative measurable functions such that $\int q {\rm d}\mu_{0} = 1$.  
For the results above, we require that we have a way to (efficiently) generate iid samples having the distribution whose density is $q$.
For the full-validity results, the only additional requirements are that we are able to (efficiently) test whether a given $x$ is in the support of $q$,
and that we have access to $\text{Oracle}(\cdot,\cdot)$ defined with respect to the set $Q$.
For the results on partial-validity, we require the ability to explicitly evaluate the function $q$ at any $x \in X$.
The results then hold as stated, and the proofs remain unchanged (overloading notation to let $q_{x}$ denote the 
value of the density $q$ at $x$, and $q(A) = \int_{A} q {\rm d}\mu_{0}$ the measure of $A$ 
under the probability measure whose density is $q$).

\subsection{Infinite Families of Distributions}
\label{sec:infinite-families}

It is also possible to extend all of the above results to \emph{infinite} families $Q$, 
expressing the sample complexity requirements in terms of the \emph{VC dimension} (\cite{VapnikC74}) of the supports $\vc = {\rm VCdim}(\{\Supp(q) : q \in Q\})$, 
and the \emph{fat-shattering dimension} (\cite{AlonBCH97}) of the family of loss-composed densities $\fat(\epsilon) = {\rm fat}_{\epsilon}(\{ x \mapsto L(q_{x}) : q \in Q \})$.

We recall the definitions of these two concepts:
\begin{definition}
Let $\mathcal{F}$ be a collection of functions which map $\mathcal{X}$ into $\{0,1\}$.
A set $X = (x_1, \dots, x_n) \subseteq \mathcal{X}$ is said to be \emph{shattered} if for every mapping $g : X \rightarrow \{0,1\}$ there exists $f_g \in \mathcal{F}$ such that $f_g(x_i) = g(x_i)$.
The \emph{VC dimension of $\mathcal{F}$}, denoted ${\rm VCdim}(\mathcal{F})$, is the largest $n$ such that there exists a set $X$ of cardinality $n$ that is shattered, and $\infty$ if no such $n$ exists.
Also, the VC dimension ${\rm VCdim}(\mathcal{S})$ of a collection $\mathcal{S}$ of sets $S \subseteq \mathcal{X}$ is defined as the VC dimension of the corresponding set of indicator functions.
\end{definition}

\begin{definition}
Let $\mathcal{F}$ be a collection of functions which map $\mathcal{X}$ into $\mathbb{R}$.
A set $X = (x_1, \dots, x_n) \subseteq \mathcal{X}$ is said to be \emph{fat-shattered to width $\epsilon$} if 
there exists $v : X \rightarrow \mathbb{R}$ such that, 
for every mapping $g : X \rightarrow \{0,1\}$ there exists $f_g \in \mathcal{F}$ and such that $f_g(x_i) \geq v(x_i) + \epsilon$ if $g(x_i) = 1$, and $f_g(x_i) \leq v(x_i) - \epsilon$ if $g(x_i) = 0$.
The \emph{fat-shattering dimension of $\mathcal{F}$ of width $\epsilon$}, denoted ${\rm fat}_{\epsilon}(\mathcal{F})$, 
is the largest $n$ such that there exists a set $X$ of cardinality $n$ that is fat-shattered to width $\epsilon$, and $\infty$ if no such $n$ exists.
\end{definition}

In this case, in the context of the full-validity results, 
for simplicity we assume that in the evaluations of $\text{Oracle}(X_{P},X_{N})$ defined above, 
there always \emph{exists} at least one minimizer $q \in Q$ of the empirical loss with respect to $X_{P}$ 
such that $\Supp(q) \cap X_{N} = \emptyset$.\footnote{It is straightforward to remove this assumption by supposing 
$\text{Oracle}(X_{P},X_{N})$ returns a $q$ that \emph{very-nearly} minimizes the empirical loss, and handling 
this case requires only superficial modifications to the arguments.}
We then have the following result.  For completeness, we include a full proof in the appendix.

\begin{theorem}
\label{thm:vc-full-validity}
For a numerical constant $c \in (0,1]$, 
the choice of parameters 
\begin{eqnarray*}
P=\Theta\left(\frac{\fat(c \ve_{1}/M) M^{2}}{\ve_1^2} \log \frac{M}{\ve_{1}} \right), & R = \Theta \left( \frac { M } {\eps_1} \right), &  T = \Theta\left(\frac{R \vc}{\ve_2} \log \frac{1}{\ve_{2}} \right)
\end{eqnarray*}
guarantees that Algorithm~\ref{alg:full-validity} outputs w.p. $3/4$ a distribution $\hat q$ with $\Loss(\hat q) \le \Loss(q^*) + \eps_1$ and $\Inv(\hat q) \le \eps_2$ 
using $P$ samples from $p$ and $R T$ invalidity queries.
  
The algorithm runs in time polynomial in $M$, $\ve_1^{-1}$, $\ve_2^{-1}$, $\vc$, and $\fat_{\ve_{1}/256}$ assuming that queries to the optimization oracle 
can be computed in polynomial time. Moreover, sampling from the resulting distribution $\hat q$ can also be performed in polynomial time.
\end{theorem}

For partial-validity, we can also extend to infinite $Q$, 
though in this case via a more-cumbersome technique.
Specifically, let us suppose the densities $q \in Q$ are bounded by $1$ (this can be replaced by any value by varying the sample size $n_2$).
Then we consider running Algorithm~\ref{alg:partial-validity} as usual, 
except replacing Step 4 with the step 
\begin{equation*}
D = {\rm Cover}_{\ve_{2}}( \{ q \in Q | \overline\Loss(q) \leq \ell \} ),
\end{equation*} 
where for any $R \subseteq Q$, ${\rm Cover}_{\ve_{2}}(R)$ denotes a minimal subset of $R$ such that 
$\forall q \in R$, $\exists q^{\ve_{2}} \in {\rm Cover}_{\ve_{2}}(R)$ with $\int | q_{x} - q^{\ve_{2}}_{x} | \mu_{0}({\rm d}x) \leq \ve_{2}$:
that is, an $\ve_{2}$-cover of $R$ under $L_{1}(\mu_{0})$.
Let us refer to this modified algorithm as Algorithm~\ref{alg:partial-validity}$^{\prime}$.
%Now denote by $\fat_{Q}(\epsilon) = {\rm fat}_{\epsilon}(Q)$.
We have the following result. 

\begin{theorem}
\label{thm:vc-partial-validity}
Suppose that the loss function $L$ is convex.
For a numerical constant $c \in (0,1]$, 
the choice of parameters
\begin{eqnarray*}
n_1 = \Theta\left(\frac{\fat(c\ve_{1}/M) M^{2}}{\ve_1^2} \log \!\left(\frac{M}{\ve_{1}}\right)  \right), & n_2 = \Theta\left(\frac {M^2 {\rm fat}_{c\ve_{2}}(Q)} {\ve_1^2 \ve_2^2} \log^{2} \!\left(\frac{M {\rm fat}_{c\ve_{2}}(Q)}{\ve_1 \ve_2}\right)\right)
\end{eqnarray*}
guarantees that Algorithm~\ref{alg:partial-validity}$^{\prime}$ (with parameters $\eps_{1}$, $\eps_{2}$, and $\alpha+\eps_{2}$) 
outputs w.p. $3/4$ a distribution with 
$\Loss(\hat q) \le \Loss(q^*) + \ve_1$ and $\Inv(\hat q) \le \alpha + 2\ve_2$ using $n_{1}$ samples from $p$ and 
$\Theta\left(\frac {M^3 {\rm fat}_{c\ve_{2}}(Q)^{2}} {\ve_1^3 \ve_2^3} \log^{3} \!\left(\frac{M {\rm fat}_{c\ve_{2}}(Q)}{\ve_1 \ve_2}\right)\right)$ 
invalidity queries.
\end{theorem}

% we'll assume, for simplicity, that there always is an empirical loss minimizer.  it just simplifies the arguments, but is pretty easy to extend to the general case by putting more \epsilon's all over the place.

\bibliographystyle{alpha}
\bibliography{biblio}

\newcommand{\etalchar}[1]{$^{#1}$}
\begin{thebibliography}{ABDCBH97}

\bibitem[ABDCBH97]{AlonBCH97}
Noga Alon, Shai Ben-David, Nicolo Cesa-Bianchi, and David Haussler.
\newblock Scale-sensitive dimensions, uniform convergence, and learnability.
\newblock {\em Journal of the ACM}, 44(4):615--631, 1997.

\bibitem[Ang88]{Angluin88}
Dana Angluin.
\newblock Queries and concept learning.
\newblock {\em Machine Learning}, 2(4):319--342, 1988.

\bibitem[Ang92]{Angluin92}
Dana Angluin.
\newblock Computational learning theory: Survey and selected bibliography.
\newblock In {\em Proceedings of the 24th Annual ACM Symposium on the Theory of
  Computing}, STOC '92, pages 351--369, New York, NY, USA, 1992. ACM.

\bibitem[ARZ18]{AroraRZ18}
Sanjeev Arora, Andrej Risteski, and Yi~Zhang.
\newblock Do {GAN}s learn the distribution? some theory and empirics.
\newblock In {\em Proceedings of the 6th International Conference on Learning
  Representations}, ICLR '18, 2018.

\bibitem[BEHW89]{BlumerEHW89}
Anselm Blumer, Andrzej Ehrenfeucht, David Haussler, and Manfred~K. Warmuth.
\newblock Learnability and the {V}apnik-{C}hervonenkis dimension.
\newblock {\em Journal of the ACM}, 36(4):929--965, 1989.

\bibitem[Fel09]{Feldman09}
Vitaly Feldman.
\newblock On the power of membership queries in agnostic learning.
\newblock {\em Journal of Machine Learning Research}, 10(Feb):163--182, 2009.

\bibitem[GPAM{\etalchar{+}}14]{GoodfellowPMXWOCB14}
Ian Goodfellow, Jean Pouget-Abadie, Mehdi Mirza, Bing Xu, David Warde-Farley,
  Sherjil Ozair, Aaron Courville, and Yoshua Bengio.
\newblock Generative adversarial nets.
\newblock In {\em Advances in Neural Information Processing Systems 27}, NIPS
  '14, pages 2672--2680. Curran Associates, Inc., 2014.

\bibitem[Han14]{Hanneke14}
Steve Hanneke.
\newblock Theory of disagreement-based active learning.
\newblock {\em Foundations and Trends{\textregistered} in Machine Learning},
  7(2--3):131--309, 2014.

\bibitem[Hau92]{Haussler92}
David Haussler.
\newblock Decision theoretic generalizations of the {PAC} model for neural net
  and other learning applications.
\newblock {\em Information and Computation}, 100(1):78--150, 1992.

\bibitem[Jac97]{Jackson97}
Jeffrey~C. Jackson.
\newblock An efficient membership-query algorithm for learning {DNF} with
  respect to the uniform distribution.
\newblock {\em Journal of Computer and System Sciences}, 55(3):414--440, 1997.

\bibitem[JM09]{JurafskyM09}
Dan Jurafsky and James~H. Martin.
\newblock {\em Speech and Language Processing: An Introduction to Natural
  Language Processing, Computational Linguistics, and Speech Recognition}.
\newblock Prentice Hall, 2009.

\bibitem[JWP{\etalchar{+}}18]{JanzWPKH18}
David Janz, Jos van~der Westhuizen, Brooks Paige, Matt~J. Kusner, and
  Jos{\'e}~Miguel Hern{\'a}ndez-Lobato.
\newblock Learning a generative model for validity in complex discrete
  structures.
\newblock In {\em Proceedings of the 6th International Conference on Learning
  Representations}, ICLR '18, 2018.

\bibitem[Kar15]{Karpathy15}
Andrej Karpathy.
\newblock The unreasonable effectiveness of recurrent neural networks.
\newblock \url{http://karpathy.github.io/2015/05/21/rnn-effectiveness/}, May
  2015.

\bibitem[KMR{\etalchar{+}}94]{KearnsMRRSS94}
Michael Kearns, Yishay Mansour, Dana Ron, Ronitt Rubinfeld, Robert~E. Schapire,
  and Linda Sellie.
\newblock On the learnability of discrete distributions.
\newblock In {\em Proceedings of the 26th Annual ACM Symposium on the Theory of
  Computing}, STOC '94, pages 273--282, New York, NY, USA, 1994. ACM.

\bibitem[KPHL17]{KusnerPH17}
Matt~J. Kusner, Brooks Paige, and Jos{\'e}~Miguel Hern{\'a}ndez-Lobato.
\newblock Grammar variational autoencoder.
\newblock In {\em Proceedings of the 34th International Conference on Machine
  Learning}, ICML '17, pages 1945--1954. JMLR, Inc., 2017.

\bibitem[KSS94]{KearnsSS94}
Michael~J. Kearns, Robert~E. Schapire, and Linda~M. Sellie.
\newblock Towards efficient agnostic learning.
\newblock {\em Machine Learning}, 17(2--3):115--141, 1994.

\bibitem[LL13]{LabbeL13}
Cyril Labb{\'e} and Dominique Labb{\'e}.
\newblock Duplicate and fake publications in the scientific literature: How
  many {SCIgen} papers in computer science?
\newblock {\em Scientometrics}, 94(1):379--396, 2013.

\bibitem[MV03]{MendelsonV03}
Shahar Mendelson and Roman Vershynin.
\newblock Entropy and the combinatorial dimension.
\newblock {\em Inventiones Mathematicae}, 152(1):37--55, 2003.

\bibitem[Ros00]{Rosenfeld00}
Ronald Rosenfeld.
\newblock Two decades of statistical language modeling: Where do we go from
  here?
\newblock {\em Proceedings of the IEEE}, 88(8):1270--1278, 2000.

\bibitem[Val84]{Valiant84}
Leslie~G. Valiant.
\newblock A theory of the learnable.
\newblock {\em Communications of the ACM}, 27(11):1134--1142, 1984.

\bibitem[VC74]{VapnikC74}
Vladimir Vapnik and Alexey Chervonenkis.
\newblock {\em Theory of Pattern Recognition}.
\newblock Nauka, 1974.

\bibitem[WNC87]{WittenNC87}
Ian~H. Witten, Radford~M. Neal, and John~G. Cleary.
\newblock Arithmetic coding for data compression.
\newblock {\em Communications of the ACM}, 30(6):520--540, 1987.

\end{thebibliography}

\appendix

%\appendix

\section{Proofs for Infinite Families of Distributions}
\label{sec:vc-classes-proofs}

The proofs of the results on handling infinite $Q$ sets follow analogously 
to the original proofs for finite $|Q|$, but with a few modifications to make 
use of results from the learning theory literature on infinite function classes.
For completeness, we include the full details of these proofs here.

\subsection{Proof of Theorem~\ref{thm:vc-full-validity}}

We begin with the proof of Theorem~\ref{thm:vc-full-validity}.
As above, we consider two key lemmas.

\begin{lemma}
\label{lem:vc-full-validity-inv}
For $P$, $R$, and $T$ as in Theorem~\ref{thm:vc-full-validity}, 
the distribution returns by Algorithm~\ref{alg:full-validity} satisfies $\Inv(\hat{q}) \leq \eps_{2}$ with probability at least $7/8$.
\end{lemma}
\begin{proof}
Following the original proof above, 
let $\text{Invalid} = \{x : \Inv(x) = 1 \}$ be the set of invalid points.
Consider $q^i$ for some $i$ and any distribution $q \in Q$. 
If $q^i( \Supp(q) \cap \text{Invalid} ) \ge \frac { \eps_2 } { R }$, 
then with probability at least $\frac { \eps_2 } { R }$ a sample generated from $q^i$ lies in $\Supp(q) \cap \text{Invalid}$. 
Furthermore, we note that the VC dimension of the collection of sets $\{ \Supp(q) \cap \text{Invalid} : q \in Q \}$ is at most $\vc$.
Thus, with $T = \Theta( \frac  { R \vc } { \eps_2 } \log \frac{1}{\eps_2} )$ samples from $q^i$, 
the classic sample complexity result from PAC learning \cite{VapnikC74,BlumerEHW89} implies % number of random samples to get an eps-net.
that with probability at least $1 - \frac{1}{8 R}$, 
every $q \in Q$ with $q^i( \Supp(q) \cap \text{Invalid} ) \geq \frac{\eps_{2}}{R}$ 
has at least one of the $T$ samples in $\Supp(q) \cap \text{Invalid}$. 
By a union bound, this holds for all $i$ in the algorithm.  Suppose this event holds.

In particular, this implies that if the algorithm returns in Step 9, so that the returned distribution $\hat{q} = q^i$ for some $i$, 
then $\Inv(q^i) = q^i( \Supp(q^i) \cap \text{Invalid} ) < \frac { \eps_2 } { R } \le \eps_2 $ as required.
Furthermore, if the algorithm returns in Step 16 instead, then the above event implies that for every $i,j$ with $i < j$, 
$q^i( \Supp(q^i) \cap \text{Invalid} ) < \frac { \eps_2 } { R }$.
Therefore, if we fix the value of $i$ selected in Step 14, we have that 
 \begin{align*}
    \Inv(\hat q) 
&= \Exp_{x\sim\hat q}\left[ \Inv(x) \right] \\
&= \Exp_{x\sim q^i}\left[ \Inv(x) \cdot \Ind\left[\exists j > i: x \in \Supp(q^j) \right]\right] \\
&\le  \sum_{j=i+1}^R \Exp_{x\sim q^i}\left[ \Inv(x) \cdot \Ind\left[ x \in \Supp(q^j) \right] \right] \\  
&= \sum_{j=i+1}^R q^i( \Supp(q^j) \cap \text{Invalid} ) \le  \sum_{j=i+1}^R \frac { \eps_2 } { R } <  \eps_2.
  \end{align*}
\end{proof}

\begin{lemma}
\label{lem:vc-full-validity-loss}
For $P$, $R$, and $T$ as in Theorem~\ref{thm:vc-full-validity}, 
the distribution $\hat{q}$ returned by Algorithm~\ref{alg:full-validity} satisfies $\Loss(\hat q) \le \Loss(q^*) + \eps_1$ with probability at least $7/8$.
\end{lemma}
\begin{proof}
Combining Corollary 2 of \cite{Haussler92} with Theorem 1 of \cite{MendelsonV03}, 
we conclude that for $P=\Theta\left(\frac{\fat(c\ve_{1}/M) M^2}{\ve_1^2} \log \frac{M}{\ve_{1}} \right)$ samples from $p$, 
we have that the empirical loss $\overline \Loss(q) \in \Loss(q) \pm \frac {\eps_1} 4$ simultaneously for all $q \in Q$ with probability at least $15/16$. 
From here on, let us suppose this event occurs.

In that case, it must be that $\overline \Loss(q^i) \le \overline \Loss(q^*)$. 
This is because the algorithm terminates if ever $q^i = q^*$ since $q^*$ generates no invalid samples,
and yet no $q^i$ with $\overline \Loss(q^i) > \overline \Loss(q^*)$ will be considered before examining $q^*$.

This implies that at any point, we have that $\Loss(q^i) \le \overline \Loss(q^i) + \frac {\eps_1} 4 \le \overline \Loss(q^*) + \frac {\eps_1} 4 \le \Loss(q^*) + \frac {\eps_1} 2$. 

Therefore, in the case that the distribution that is output is $\hat q = q^i$ it will satisfy the given condition.
To complete the proof we show the required property when returned distribution $\hat q$ is the improper meta-distribution.
  
In that case, we have that:
    \begin{align*}
      \Loss(\hat q) 
  &\le \Exp_{x\sim p}\left[ L\left(q_x^i \cdot \Ind\left[\exists j > i: x \in \Supp(q^j) \right] \right) \right] \\
  &\le \Loss(q^i) + M \cdot \Pr_{x\sim p}\left[ x \in \Supp(q^i) \wedge \forall j > i: x \notin \Supp(q^j) \right] \\
  &\le \Loss(q^*) + \frac {\eps_1} 2 + M \cdot \Pr_{x\sim p}\left[ x \in \Supp(q^i) \wedge \forall j > i: x \notin \Supp(q^j) \right]
    \end{align*}
However, since a random index $i \sim \Unif(\{1,...,R\})$ is chosen, we have that in expectation over this random choice
    \begin{align*}
\Exp_i&\left[ \Pr_{x\sim p}\left[ x \in \Supp(q^i) \wedge \forall j > i: x \notin \Supp(q^j) \right] \right] \\
&= \frac 1 R \sum_{i=1}^R \Pr_{x\sim p}\left[ x \in \Supp(q^i) \wedge \forall j > i: x \notin \Supp(q^j) \right] \\
&= \frac 1 R \Exp_{x\sim p} \left[ \sum_{i=1}^R \Ind\left[ x \in \Supp(q^i) \wedge \forall j > i: x \notin \Supp(q^j) \right] \right] 
\le \frac 1 R
   \end{align*}
  where the last inequality follows since $\sum_{i=1}^R \Ind\left[ x \in \Supp(q^i) \wedge \forall j > i: x \notin \Supp(q^j) \right] \le 1 $ as only the largest $i$ with $x \in \Supp(q^i)$ has that for all $j > i$, $x \notin \Supp(q^j)$.

By Markov's inequality, we have that with probability at least $15/16$, 
a random $i$ will have $$\Pr_{x\sim p}\left[ x \in \Supp(q^i) \wedge \forall j > i: x \notin \Supp(q^j) \right] \le \frac {16} R.$$

Therefore, the choice of $R = 32 \frac M {\ve_1} = \Theta \left( \frac { M } {\eps_1} \right)$ guarantees that $\Loss(\hat q) \le \Loss(q^*) + {\eps_1}$.
The overall failure probability is at most $1/16+ 1/16 = 1/8$.
\end{proof}

\begin{proof}[Proof of Theorem~\ref{thm:vc-full-validity}]
Theorem~\ref{thm:vc-full-validity} follows immediately from the above two lemmas by a union bound.
\end{proof}

\subsection{Proof of Theorem~\ref{thm:vc-partial-validity}}

Next, the proof of Theorem~\ref{thm:vc-partial-validity} follows similarly to the original proof 
of Theorem~\ref{thm:partial-validity}, with a few important adjustments.
As in the statement of the theorem, we consider running Algorithm~\ref{alg:partial-validity}$^{\prime}$ 
with parameters $\eps_{1}$, $\eps_{2}$, and $\alpha+\eps_{2}$.
As in the proof of Theorem~\ref{thm:partial-validity}, we proceed by establishing three key lemmas.
As much of this proof essentially follows by \emph{plugging in} the altered set $D$ (from the new Step 4) 
to the arguments of the original proofs above, in the proofs of these lemmas we only highlight the reasons 
for which this substitution remains valid and yields the stated result.

\begin{lemma}
\label{lem:vc-partial-validity-inner-loop}
With probability at least $14/15$, the loop at Line~\ref{ln:partial-validity-inner-loop} of Algorithm~\ref{alg:partial-validity}$^{\prime}$ 
requires at most $O\left(\frac{{\rm fat}_{c\eps_2}(Q)}{\eps_2} \log\!\left( \frac{1}{\eps_{2}} \right) \right)$ iterations for each $\ell$.
\end{lemma}
\begin{proof}
We invoke the original argument from the proof of Lemma~\ref{lem:partial-validity-inner-loop} verbatim, 
except that rather than bounding the initial size $|D|$ in Step 4 by $|Q|$, 
we use the fact that Step 4 in Algorithm~\ref{alg:partial-validity}$^{\prime}$ initializes $|D|$ 
to the minimal size of an $\eps_{2}$-cover of $\{q \in Q | \overline\Loss(q) \leq \ell \}$, 
which is at most the size of a minimal $\eps_{2}$-cover of $Q$ (under the $L_{1}(\mu_{0})$ pseudo-metric).
Thus, Theorem 1 of \cite{MendelsonV03} implies that, for every $\ell$, this initial set $D$ satisfies 
\begin{equation}
\label{eqn:mv-cover-bound}
\log(|D|) = O\!\left( {\rm fat}_{c\eps_{2}}(Q) \log\!\left( \frac{1}{\eps_{2}} \right) \right).
\end{equation}
The lemma then follows from the same argument as in the proof of Lemma~\ref{lem:partial-validity-inner-loop}.
\end{proof}

\begin{lemma}
\label{lem:vc-partial-validity-invalid}
With probability at least $14/15$, if at any step a distribution $\mu'_D$ is output, $\Inv(\mu'_D) \le \alpha + 2\eps_2$.
\end{lemma}
\begin{proof}
The argument remains identical to the proof of Lemma~\ref{lem:partial-validity-invalid}, 
except again substituting for $\log|Q|$ the quantity on the right hand side of \eqref{eqn:mv-cover-bound}, 
and substituting $\alpha+\eps_{2}$ for $\alpha$.
\end{proof}

\begin{lemma}
\label{lem:vc-partial-validity-loss}
With probability at least $14/15$, if at any step a distribution $\mu'_D$ is output, $\Loss(\mu'_D) \le \ell + 2\eps_1/3$, where $\ell$ is the step at which the distribution was output.
\end{lemma}
\begin{proof}
Combining Corollary 2 of \cite{Haussler92} with Theorem 1 of \cite{MendelsonV03} implies that 
the choice $n_{1} = \Theta\left(\frac{\fat(c\ve_{1}/M) M^{2}}{\ve_1^2} \log \!\left(\frac{M}{\ve_{1}}\right)  \right)$ 
suffices to guarantee every $q \in Q$ has $\overline\Loss(q)$ within $\pm \eps_{1}/3$ of $\Loss(q)$.
Substituting this argument for the final step in the proof of Lemma~\ref{lem:partial-validity-loss}, 
and leaving the rest of that proof intact, this result follows.
\end{proof}

\begin{proof}[Proof of Theorem~\ref{thm:vc-partial-validity}]
The proof of Theorem~\ref{thm:vc-partial-validity} concludes by observing that, upon reaching $\ell$ within $\eps_{1}/3$ of $\Loss(q^*)$ (where $q^{*}$ is the optimal distribution), 
the closest (in $L_{1}(\mu_{0})$) element $q$ of the corresponding $D$ set will have $\Inv(q) \leq \Inv(q^{*})+\eps_{2} \leq \alpha+\eps_{2}$, and (by definition of $D$) $\Loss(q) \leq \Loss(q^{*})+\eps_{1}/3$.
Thus, this $q$ will never be eliminated (assuming all estimates involving its loss and validity are accurate, 
which happens with probability at least $19/20$). % and that the loop in line~\ref{ln:partial-validity-outer-loop} steps by increments of $\eps_1/3$. 
Combining this with Lemma~\ref{lem:vc-partial-validity-loss}, if we output $\hat q$, then $\Loss(\hat q) \leq \Loss(q^*) + \eps_1$.
\end{proof}

\end{document}